%% file: main.tex
\documentclass[letterpaper]{article} 
\usepackage{aaai23}  
\usepackage{times}  
\usepackage{helvet}  
\usepackage{courier}  
\usepackage[hyphens]{url}  
\usepackage{graphicx} 
\def\UrlFont{\rm}  
\usepackage{natbib}  
\usepackage{caption} 
\usepackage{xcolor}
\usepackage{algorithm}
\usepackage{algorithmic}
\usepackage{multirow}
\usepackage{booktabs}

\usepackage{newfloat}
\usepackage{listings}
\DeclareCaptionStyle{ruled}{labelfont=normalfont,labelsep=colon,strut=off} 
\floatstyle{ruled}
\newfloat{listing}{tb}{lst}{}
\floatname{listing}{Listing}
\usepackage{amsfonts}
\usepackage{amsmath}

\usepackage{amsthm}

\newtheorem{defn}{Definition} 
\newtheorem{prop}{Proposition}

\title{Planning for Attacker Entrapment in Adversarial Settings}
\author {
    Brittany Cates\textsuperscript{\rm 1},
    Anagha Kulkarni\textsuperscript{\rm 2},
    Sarath Sreedharan\textsuperscript{\rm 1}
}
\affiliations {
    \textsuperscript{\rm 1} Department of Computer Science, Colorado State University, CO, USA\\
    \textsuperscript{\rm 2} Invitae Corporation, San Francisco, CA, USA\\
    brittany.cates@colostate.edu, anagha.kulkarni@invitae.com, sarath.sreedharan@colostate.edu
}

\usepackage{bibentry}

\begin{document}
\include{commands}

\maketitle

\begin{abstract}
\input{Sections/0-abstract.tex}

\end{abstract}

\input{Sections/1-introduction}
\input{Sections/3-motivation}
\input{Sections/2-background}
\input{Sections/4-0-method}

\input{Sections/4-1-budget}

\input{Sections/5-evalution}
\input{Sections/6-relatedwork}

\input{Sections/7-discussion}
\input{Sections/8-ack}

\bibliography{bib}

\end{document}

%% file: commands.tex
\newcommand{\Shortcite}[1] {\citeauthor{#1}~\shortcite{#1}}
\newcommand{\ShortciteManySameLead}[2] {\citeauthor{#1}~\shortcite{#1,#2}}

\newcommand{\Attacker}              {\ensuremath{\mathcal{A}}}
\newcommand{\Defender}              {\ensuremath{\mathcal{D}}}
\newcommand{\AModel}              {\ensuremath{\mathcal{M}^{\Attacker}}}
\newcommand{\DModel}              {\ensuremath{\mathcal{M}^{\mathcal{D}}}}
\newcommand{\MZ}              {\ensuremath{\mathcal{M}^{0}}}

\newcommand{\AStates}              {\ensuremath{S^{\Attacker}}}
\newcommand{\AActs}              {\ensuremath{A^{\Attacker}}}
\newcommand{\ATran}              {\ensuremath{T^{\Attacker}}}
\newcommand{\ARe}              {\ensuremath{R^{\Attacker}}}
\newcommand{\ADisc}              {\ensuremath{\gamma^{\Attacker}}}
\newcommand{\AVal}              {\ensuremath{\mathcal{V}^{\Attacker}}}
\newcommand{\AValStar}              {\ensuremath{\mathcal{V}^{\Attacker*}}}
\newcommand{\AQVal}              {\ensuremath{\mathcal{Q}^{\Attacker}}}
\newcommand{\AQValStar}              {\ensuremath{\mathcal{Q}^{\Attacker*}}}
\newcommand{\QVal}              {\ensuremath{\mathcal{Q}}}
\newcommand{\QValStar}              {\ensuremath{\mathcal{Q}^*}}

\newcommand{\DStates}              {\ensuremath{S^{\Defender}}}
\newcommand{\DActs}              {\ensuremath{A^{\Defender}}}
\newcommand{\DTran}              {\ensuremath{T^{\Defender}}}
\newcommand{\DRe}              {\ensuremath{R^{\Defender}}}
\newcommand{\DDisc}              {\ensuremath{\gamma^{\Defender}}}
\newcommand{\DVal}              {\ensuremath{\mathcal{V}^{\Defender}}}
\newcommand{\DValStar}              {\ensuremath{\mathcal{V}^{\Defender*}}}
\newcommand{\DQVal}              {\ensuremath{\mathcal{Q}^{\Defender}}}
\newcommand{\DQValStar}              {\ensuremath{\mathcal{Q}^{\Defender*}}}

\newcommand{\Xsays}[2] {\textcolor{purple}{\textbf{#1 says:} \emph{#2}}}
\newcommand{\Ssays}[1] {\Xsays{Sarath}{#1}}

\newcommand{\todo}[1] {\textcolor{red}{#1}}
\newcommand{\newicaps}[1] {\textcolor{red}{#1}}

\newcommand{\update}[1] {\textcolor{magenta}{#1}}

%% file: Sections/0-abstract.tex
In this paper, we propose a planning framework to generate a defense strategy against an attacker who is working in an environment where a defender can operate without the attacker’s knowledge. The objective of the defender is to covertly guide the attacker to a trap state from which the attacker cannot achieve their goal.  Further, the defender is constrained to achieve its goal within $K$ number of steps. This $K$ is calculated as a pessimistic lower bound within which the attacker is unlikely to suspect a threat in the environment. Such a defense strategy is highly useful in real-world systems like honeypots or honeynets, where an unsuspecting attacker interacts with a simulated production system while assuming it is the actual production system. Typically, the interaction between an attacker and a defender is captured using game theoretic frameworks. Our problem formulation allows us to capture it as a much simpler infinite-horizon discounted Markov Decision Process (MDP), in which the optimal policy for the MDP represents the defender's strategy against the actions of the attacker. Through empirical evaluation, we show the merits of our problem formulation.

\begin{figure*}[ht!]
    \centering
    \includegraphics[width=0.60\textwidth]
    {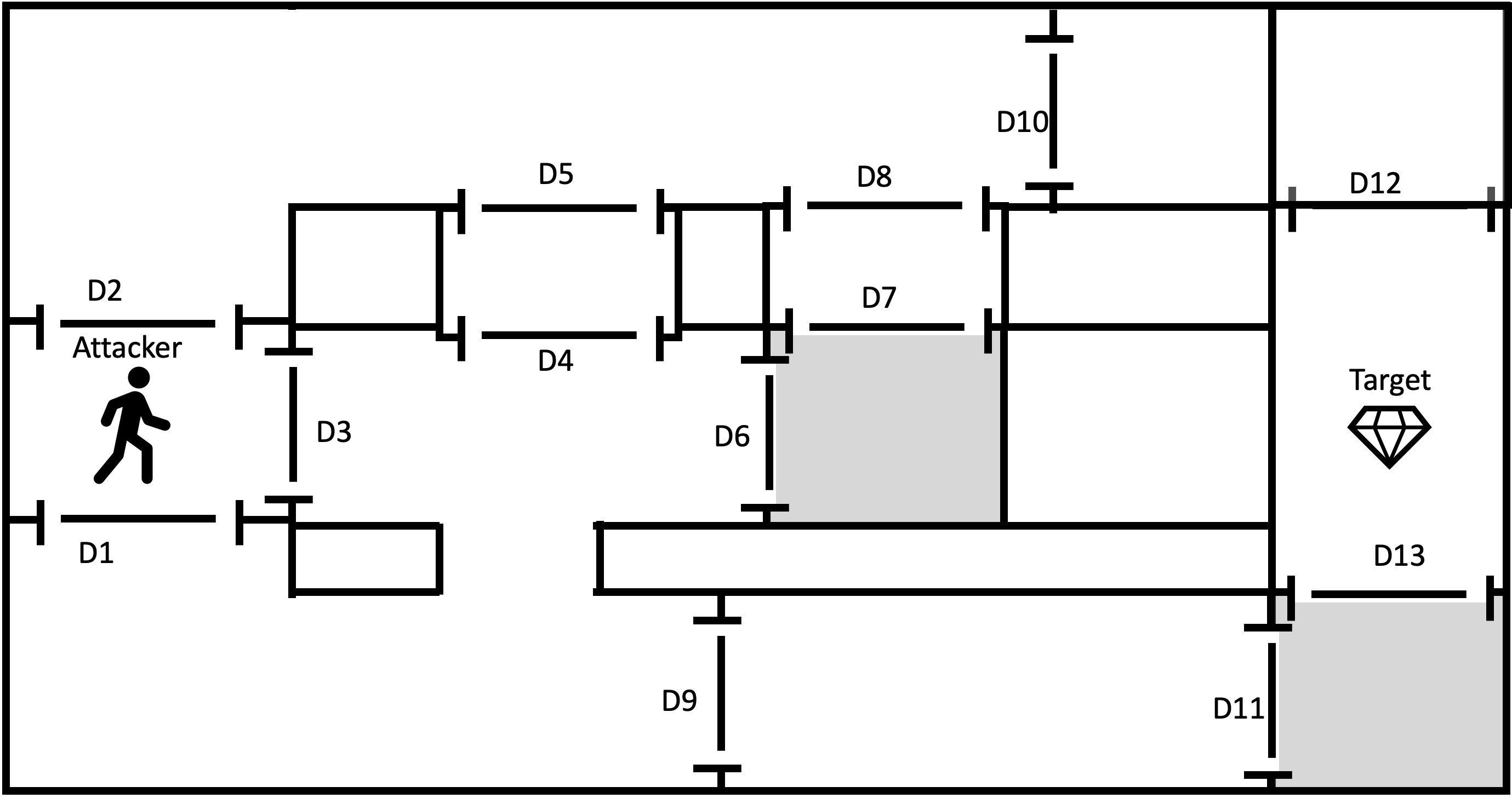}
    \caption{A diagram illustrating the running example. Here we see an attacker trying to break into a specific room to steal an object. The defender's goal is to get the attacker to enter one of the trap rooms highlighted in grey by controlling which of the automatic doors open when the attacker enters a room.}
    \label{fig:map}
\end{figure*}

%% file: Sections/1-introduction.tex
\section{Introduction}
Adversarial planning is a topic that has received significant attention in recent years \cite{masters2021extended, bernardini2020optimization, chakraborti2019explicability}. It has been considered in a diverse range of settings, from simple path planning \cite{masters2017deceptive} to more realistic cybersecurity problems \cite{letchford2013optimal}. These works have not only provided the community with an opportunity to showcase the effectiveness of modern planning tools but to develop new planning paradigms better suited to handle such cases (for example, the Stackelberg planning formulation \cite{speicher2018stackelberg}). An underlying theme in many of the above-mentioned examples is that the problem is framed as essentially game-theoretic in nature. Through this paper, we hope to introduce a new class of adversarial planning problems of practical import, which lends itself to a simple and intuitive planning formulation.

In this paper, we introduce the {\em attacker entrapment problem}, which involves a hidden defender manipulating its environment to undermine an attack. We will conceptualize this as a problem of leading an attacker, oblivious to the defender’s presence, to predetermined trap states. The effectiveness of the defender relies on its ability to keep its presence hidden from the attacker. As long as the attacker is unaware of the defender’s existence, the defender can rely on basic planning techniques to identify the defender's strategies.

While this may seem like an unrealistic setting, it is in fact a widely studied problem in the domain of cybersecurity. Specifically, honeypots are a strategy often used in real-world networks to address attacker entrapment problems. As the name suggests, honeypots are used to lure and trap black hat attackers by simulating a real production system that the attackers can penetrate and interact with. Honeypots can be used as a defense mechanism to divert the attacker from a real system and trap them, or they can be used to study an attacker's strategies in order to develop stronger security mechanisms. 

It is worth noting that current strategies for developing honeypots are completely offline processes. The network designer is tasked with organizing a system so that honeypots corresponding to attractive targets are placed in a potential attacker’s path. From the perspective of an attacker entrapment problem, this represents a mechanism or environment design solution \cite{keren2021goal}. We instead propose a novel, active solution, which uses a planner to adapt to potential actions taken by the attacker. 

We will capture the planning problem of the defender as a Markov Decision Process (MDP). While concealing its presence, the defender will use the attacker’s beliefs about the system architecture to predict the attacker's strategy and lead the attacker to a predetermined trap state. It will mask its own actions using the attacker’s beliefs about the stochasticity of the domain. Unbeknownst to the attacker, the defender can control the outcome of certain attacker actions. For example, the defender could decide what password should be returned to the attacker in response to an SQL injection attack, or when considering the security of a physical space, which automatic door to open when the attacker reaches a room. By carefully selecting the outcomes, the defender can lead the attacker to a trap state or failing that, lead them to a state with lower utilities.

In addition to introducing this problem framework and a solution technique, we will introduce a way to calculate a budget for defender actions. We will show how such a budget allows the defender to manipulate outcomes without arousing the attacker’s suspicion that the environment is being controlled. We will evaluate our solution on several traditional MDP benchmarks and present the computational characteristics of our method.

To summarize, the contributions of this paper include:
\begin{enumerate}
    \item Identification and formalization of the attacker entrapment problem.
    \item Formulation of a solution to the attacker entrapment problem as an MDP planning problem.
    \item Development of search-based methods to identify a lower bound on the defender budget that is guaranteed to avoid detection from the attacker.
    \item Evaluation of the effectiveness of solution methods in reducing the attacker values on a number of standard MDP benchmarks.
\end{enumerate}

%% file: Sections/3-motivation.tex
\section{Motivating Example}

As a running example, we consider a simple scenario involving the security of a small showroom. The attacker takes the form of a thief, who is attempting to steal a valuable diamond from the showroom. The thief has access to a map of the showroom and knows the exact location of the valuable item. The only source of stochasticity within this domain is a set of automatic doors spread throughout the showroom. The attacker is under the impression that whenever they enter a room with multiple doors, one of the doors will open with uniform probability. However, unbeknownst to the attacker, there exists an automated defender agent that can control which doors will open. The objective of the defender is to choose the doors which will lead the attacker into one of the trap states. While the attacker knows these rooms exist, they do not know that these are trap states. 

Let Figure \ref{fig:map} correspond to the layout of the showroom and the current position of the attacker,  with trap states highlighted in grey. The defender could choose to open door D3. If the attacker is a rational agent, they would simply follow the corridor until they enter the trap state. However, this may not be the case. They may wander into the lower corridor instead. The defender would need to adapt by opening door D11, leading the attacker to another trap state. The defender will have to keep track of where the attacker may be going and what they may be doing, so it can respond in turn and achieve the desired outcome. In addition to directing the attacker into one of the trap states, the defender may also have other considerations. For example, some of these hallways might be lined with expensive items that the attacker could damage. To avert such costly risks, the defender can force the attacker to use paths that minimize the potential damage. 

%% file: Sections/2-background.tex
\section{Background}
We capture the planning problem as an infinite horizon discounted Markov Decision Process, or MDP \cite{puterman1990markov}. Keeping with standard notations, we will represent this MDP as a tuple of the form  $\mathcal{M} = \langle S, A, T, R, \gamma, S_0\rangle$, where $S$ is the set of states, $A$ is the set of actions, $T: S \times A \times S \rightarrow [0,1]$ is the transition function, $R: S \times A \times S \rightarrow \mathbb{R}$ is the reward function, $\gamma \in [0,1)$ is the discount factor, and $S_0$ is the initial state for the task. The solution concept used in this class of MDPs corresponds to stationary deterministic policies, henceforth referred to simply as policies. Each policy $\pi$ corresponds to a function, i.e., $\pi: S \rightarrow A$, that associates an action with each possible state. The value of a given policy $\pi$, $V^\pi: S \rightarrow \mathbb{R}$, captures the expected discounted sum of rewards obtained by executing the given policy from a state. Similarly, a $Q$-function $\QVal^\pi: S \times A \rightarrow \mathbb{R}$, captures the expected discounted sum of rewards obtained by executing a specific action in the given state and then following the policy $\pi$.
A policy is said to be optimal for $\mathcal{M}$ (denoted as $\pi^*$) if there exists no other policy $\pi'$ and state $s \in S$, such that $V^{\pi'}(s) > V^{\pi^*}(s)$. We will denote the value function corresponding to the optimal policy, using the notation $V^*$(we will similarly denote the optimal $Q$-function using the notation $\QValStar$).
We will use the term trajectory to refer to a sequence of states and actions (denoted as $\tau = \langle s_0, a_0,...,s_i\rangle$). We will use the notation $P(\tau|\mathcal{M})$, to capture the likelihood of a trajectory $\tau$ occurring under a given model $\mathcal{M}$, where we define the probability recursively as
\[P(\tau|\mathcal{M}) = T(s_0,a_0,s_1) \times P(\langle s_1, a_1,...,s_i\rangle|\mathcal{M})\]
In this paper, we will associate a different MDP with both the attacker and the defender, so we differentiate the two MDPs and their components using the superscript  $\mathcal{A}$ and $\mathcal{D}$.

%% file: Sections/4-0-method.tex
\section{Our Approach}
At the heart of our approach lies the interaction between two agents, an attacker (\Attacker) and a hidden defender (\Defender). Each agent has a model of the problem that they believe they are solving. The attacker starts with an infinite horizon discounted MDP of the form $\AModel = \langle \AStates, \AActs, \ATran, \ARe, \ADisc, S^{\Attacker}_0 \rangle$, where \AStates~ are the set of states that the attacker ascribes to task, 
\AActs~ is the set of actions the attacker can perform,
\ATran~ is the attacker's belief about how the task may evolve in response to the attacker's actions, and finally, $\ARe$ and $\ADisc$ represent the attacker's internal reward function and discounting. To simplify the formulation, we will assume that the attacker's reward is always a non-negative value.
This could even be a purposefully distorted version that the defender may have leaked to the outside world. 
We will denote the value function calculated from this model as $\AVal$ and denote the optimal value as $\AValStar$ (we will similarly use $\AQVal$ for the $Q$-function). In our running example, the attacker's states consist of the various positions the attacker can be in and the actions correspond to those for moving around the showroom and finally stealing the diamond when the attacker is in the final room.

One piece of information that the defender does not have access to is the potential policy that the attacker might follow. While the defender might not be able to infer the exact policy, the defender can still generate a distribution over possible attacker actions from the model. In particular, the attacker's decision-making can be modeled using a noisy-rational model \cite{noisyrat}. This is generally a very conservative model, that still allows probabilities for non-optimal actions. It has the additional advantage of being a psychologically feasible model, and it is thus able to model cases where we might be dealing with a human attacker. For example, the use of such noisy-rational model would allow us to capture the possibility that the attacker may wander from the optimal path in the example provided in Figure \ref{fig:map}.
In particular, we will define the probability of action selection to be

\[P(a|s) \propto e^{\kappa \times \AQValStar(s,a)}, ~\textrm{Where}~s \in \AStates ~ a \in \AActs\]

Effectively, this corresponds to a probabilistic distribution where higher value actions are given more weight, and $\kappa$ is a parameter which allows us to control how likely it is that the attacker may select a non-optimal action.

While an attacker would try to maximize the expected cumulative reward they receive, the goal of the defender is to lead the attacker to one of the possible trap states, while minimizing the total reward the attacker may receive (after all, the attacker’s reward may correspond to cases where they could be destroying or generally disrupting the operations on the target environment). 
The defender has access to two types of actions. The first is a no-op action, wherein the defender does not alter the environment and allows the attacker to execute its chosen action. In the second type of action, the defender chooses which of the potential outcomes of an attacker’s action should be applied to the environment. For our running example, this corresponds to deciding which of the doors next to the attacker should open.

Furthermore, we will place a limit on the number of times the defender gets to select a specific outcome. In addition to placing limits on what the defender can do, it also reflects the fact that if the defender keeps choosing a specific outcome for a particular attacker action in a state, it may lead to the attacker being suspicious that they may be incorrect about the model. In our running example, consider the starting position for the attacker. There are three potential doors next to the attacker, and per the defender's policy, it chooses to open door D3. If the attacker decides to simply stay in that room, they would expect one of the doors to again open in the next turn.
Let's assume that the defender keeps opening the same door, regardless of how many times the attacker chooses to stay in the room. Even if the attacker attributes the first few instances to the randomness of the environment, they may start doubting their knowledge about the environment if the same outcome keeps getting selected. Such changes in beliefs could affect the attacker's behavior, which could in turn  reduce the effectiveness of the method.

Thus, we will limit the defender's actions to the first $K$ steps. Section \ref{sec:budget} will provide a more formal description of this phenomenon and how we can calculate a pessimistic lower bound on $K$. With these component definitions in place, we are now ready to define our central problem as follows:


\begin{defn}
For a given attacker operating using a model of task  $\AModel$, in an environment with a set of trap states $S_T$ where the attacker receives a reward of zero, an {attacker entrapment problem} of budget $K$ corresponds to the problem of the defender selecting a sequence of $K$ actions (either a no-op or an outcome-selection actions), that would result in a minimal cumulative reward for the attacker.
\end{defn}

Our primary argument in this paper is that given this setting, we can encode the optimal attacker entrapment problem as another MDP. Solutions to this MDP will automatically devise a policy for the defender that will look at the actions taken by the attacker and decide how the outcomes could be shaped to get the best possible result from the defender’s point of view. This new MDP would again take the form of an infinite horizon discounted MDP. 
However, now the state space consists of the attacker state, the action being followed by the attacker and the remaining budget.
As discussed, the actions involve a no-op action and a set of actions that can choose which of the outcomes to manifest. Each action (even a no-op) will also lead to a decrease in the budget. The attacker action component of the state is determined by the likelihood of the attacker selecting that action in that state (defined by $P^{\Attacker}$). All states where the budget hits zero and the ones that are trap states are treated as absorbing states. In terms of the reward function, for every non-absorbing state the defender's reward corresponds to the negative of the reward received by the attacker. For absorbing states, the defender receives zero reward and will receive the negative of the optimal value for the attacker model. This means that the system will try to drive the state to one of the trap states while trying to pass through low-reward states for the attacker. Failing that, the objective would be to reach a state with the minimal value for the attacker. Again, all the defender actions are limited by the budget.

More formally, the defender model is given as
\[\DModel_{(S_T,K)} = \langle \DStates, \DActs, \DTran, \DRe, \DDisc, S^{\Defender}_0 \rangle,\] 
where $S_T \subseteq \AStates$ is the set of trap states. Here the states are given as $\DStates = \AStates \times \AActs \times I^K$, 
such that $I^K \subseteq \mathbb{Z} $, 
and actions as $\DActs  = \{noop\} \cup {a^\Defender_s | s\in \AStates}$. Let
 $\hat{s}$ be the state to which the attacker intends to move.
For every state $(s,a,k) \in \DStates$ where $s \not\in S_T$ and $k \neq 0$, the transition probability is defined as,
\begin{align*}
    \DTran((s,a,k), noop, (s',a',k')) = \begin{cases} \ATran(s,a,s') \times P^{\Attacker}(a'|s') \\~~~~~\textrm{if}~k'=k-1\\ 0   ~\textrm{otherwise}\end{cases}
\end{align*}

For the non-no-op actions, the transition probability is defined as 
\[\DTran((s,a,k),a^\Defender_{\hat{s}}, (s',a',k')) = 
\begin{cases} 1 ~\textrm{if}~k'=k-1~\textrm{and}~\hat{s}=s'\\ 0   ~\textrm{otherwise}\end{cases}
\]
For an absorbing state, all transition probabilities would be zero. Now for the reward, for all non-absorbing states 
$\DRe((s,a,k)) = -1 \times \ARe(s)$, 
while for a state $s \in S_T, \DRe((s,a,k)) = 0$ for all $k$ and for all states where $k=0$ and the state is not a trap state 
$\DRe((s,a,0)) =-1 \times \AQValStar(s,a)$.
By allowing the reward value to be equal to the negative of the optimal $Q$-function value of the agent, we allow for the fact that once the defender stops acting in the world, the attacker is free to achieve the maximum possible value. This adds an additional incentive for the defender to drive the attacker to lower-valued states if the defender is unable to lead the attacker to a trap state.
The discount factor is kept the same as that of the attacker. One could also associate a cost with each action the defender chooses and restrict the non-no-op actions to some subset of states, but we will skip those to keep the formulation simple. 

It is worth noting that the above formulation includes the budget as part of the state, rather than considering a finite horizon MDP. This is because under a finite horizon MDP, the reward associated with a state is not dependent on the time step. However, we want to set the reward equal to the negative of the attacker's $Q$-function value in states where the defender's budget is zero. It is unclear how easily we can capture this using finite horizon MDPs.

It is also worth noting that depending on the scenarios, the ability of the defender to determine the outcome might be restricted to some set of states or actions. Or there might be some cost associated with enforcing some specific outcome. The above formulation can easily be modified to capture such consideration. However, we will go with this less restricted and simpler setting, to allow for a cleaner formulation and to potentially derive guarantees on possible defender behaviors.
In particular, we can use this formulation to establish a relationship between the optimal value function of the defender MDP and the attacker MDP.

\input{Sections/4-0-1-proof}

Note that for an attacker starting from a given state $s$ and performing an action $a$, the expected value the attacker would obtain would be upper bounded by $|\DValStar((s,a,K))|$. This follows intuitively from the fact that for all non-terminal states, the reward associated with a state in $\DModel$ is just the negative of the reward received by the attacker. For a terminal state, if it is a trap state, the attacker would in fact receive a reward of zero.
If the terminal state corresponds to a state where the budget is zero, then the defender receives the negative of the optimal $Q$-function value for the action taken by the attacker. This is an upper bound on the value the attacker could receive from that state onward.
The above proposition establishes that an optimal policy derived using the above formulation, will never result in the attacker receiving a higher value. In fact, as we will see in the evaluation section, our method is generally very effective at reducing the total expected value received by the attacker.

%% file: Sections/4-0-1-proof.tex
\begin{prop}
    For any given state $s \in \AStates$, the optimal value of any corresponding state $(s,a,k)$, is given as $|\DValStar((s,a,k))| \leq \AQValStar(s,a)$, for all attacker actions $a$ and current budget value $k$.
\end{prop}
\begin{proof}[Proof]
We can easily show this relation through induction, by proving a stronger relation, namely $\DValStar((s,a,k)) \geq -1\times \AQValStar(s,a)$. This is equivalent to stating that 
 $-1\times\DValStar((s,a,k)) \leq  \AQValStar(s,a)$ and since $\DValStar(s,a) \leq 0$ (a result of the assumption about the original $\ARe$), the relation  $|\DValStar((s,a,k))| \leq \AQValStar(s,a)$ holds when you consider the absolute values.
The relation $\DValStar((s,a,k)) \geq -1\times \AQValStar(s,a)$ trivially holds when $k=0$.
Assuming that the relation holds for any state with budget $k=i$, let us consider the value of a state for budget $k=i+1$. Following the Bellman equations, the value will be defined as
\begin{align*}
   \DValStar((s,a,i+1)) =\\ \DRe((s,a,i+1)) + \\ \textrm{max}_{a_{\Defender}}(\sum_{(s',a',i)}(\DTran((s,a,i+1)), a_{\Defender}, (s',a',i))  \times \\\DValStar((s',a',i)) 
\end{align*}

Now we know that the relation holds for states with budget $i$, thus $\DValStar((s',a',i)) \geq -1
\times \AQValStar((s',a')) \geq -1
\times  \AValStar((s'))$, thus
\begin{align*}
\DValStar((s,a,i+1)) \geq \\ \DRe((s,a,i+1)) + \\\textrm{max}_{a_{\Defender}}(\sum_{(s',a',i)}(\DTran((s,a,i+1)), a_{\Defender}, (s',a',i))\times \\ -1 \times \AValStar(s'))
\end{align*}

We now have two possible candidates for $a_{\Defender}$ with the maximum value. Either it is a deterministic selection of an outcome, or it is a no-op action. If it is the former, the Bellman operator would select the state with the lowest  $\AValStar$, say $\Bar{s}$, thus we have
\begin{align*}
    \DValStar((s,a,i+1)) \geq -1 \times (\ARe(s) +   \AValStar(\Bar{s}))\\ \geq -1 \times \AQValStar(s,a)
\end{align*}
where, 
\begin{align*}
    \AQValStar(s,a) = \ARe(s) + \sum_{\hat{s}}   \AValStar(\hat{s}))
\end{align*}
The fact 
\[-1 \times(\ARe(s) +   \AValStar(\Bar{s})) \geq -1 \times \AQValStar(s,a)\] follows from the fact that
\[\ARe(s) +   \AValStar(\Bar{s}) \leq \AQValStar(s,a)\]
which, in turn, follows from the fact that the lowest value of $\AValStar$ would be less than the weighted sum of the $\AValStar$ for all other reachable states. This proves that the required relation holds when the action that provides the maximum value is an outcome selection action. In the case of a no-op action, the Bellman equation for $\DValStar((s,a,i+1))$ directly maps to $-1\times\AQValStar((s,a))$, as the action results in the same transitions that would have occurred in the attacker model. This proves the original result.
\end{proof}

%% file: Sections/4-1-budget.tex
\subsection{Budget}
\label{sec:budget}
The previously discussed formulation will never generate a state action sequence that is not possible under the model used by the attacker. In theory, the attacker would not be surprised by the outcomes observed, regardless of the choices made by the defender. However, our choice to restrict the defender’s actions to the first $K$ steps is a reflection of the fact that this would hardly ever be true in practice. A more realistic assumption is that the attacker may maintain some uncertainty regarding their own beliefs about the environment model. To approximate how the attacker’s beliefs about the model may evolve, we will assume that the attacker is a Bayesian reasoner, an assumption that would fit well even when the attacker may be human. Under this approximation, the attacker will maintain two hypotheses about the environment: either the environment model is the same as $\AModel$ or that they do not understand the system dynamics at all. Effectively, the second hypothesis  corresponds to a case where the human believes the underlying model could be anything other than the one specified by $\AModel$. This is closely related to problems like new species induction \cite{zabell1992predicting} and open-world reasoning \cite{senator2019science}. However, as works like \cite{Sreedharan0SK21} have shown, one could approximate this by using a high entropy model as the basis for this second hypothesis. In this second hypothesis, the human believes that transition to any state is possible from any given state when any action is executed.  Keeping with the previous literature, we will represent this model as $\MZ= \langle \AStates, \AActs,T^0,R^0, \ADisc \rangle$, such that for all $s,s' \in  \AStates$ and $a \in \AActs$, we have
\[T^0(s,a,s')  = \frac{1}{|\AStates|} \] 

Since we are interested in worst-case scenarios, we will assume a truly unbiased attacker whose belief about the model is equally distributed across these two hypotheses. Note that assuming that the attacker is an agent that is operating under this uncertainty does not shift the policy of the attacker from what was defined in earlier sections, as under $\MZ$ all policies are equivalent. As such, the attacker's actions are determined by $\AModel$.

The objective of the defender would be to prevent the attacker from placing more weight on the hypothesis that they are incorrect about the model. If they start believing that they could be wrong about the model before they reach a trap state, they could take corrective or evasive actions.
One could adopt a belief MDP style formulation \cite{KaelblingLC98}, where we track the exact belief of the attacker. However, the setting provides the opportunity to leverage more efficient formulations, which still guarantees the defender does not tip the attacker’s belief about whether their understanding of the underlying model is incorrect.
The budget represents, in the worst case, the minimal number of transitions that can occur from the initial state, wherein the attacker's beliefs may tip over to $\MZ$. More formally, we can now define the budget $K$ as follows:

\begin{defn}
For a given pair of attacker models $\AModel$ and $\MZ$, the defender action budget is given as $K$, if 
\begin{enumerate}
\item There exists a trajectory $\tau=\langle s_0,a_0,....,s_K \rangle$ of length $K$ (i.e., $|\tau| = K$), such that $s_0 = S^{\Attacker}_0$, $P(\tau|\AModel) > 0$, and $P(\tau|\AModel) < P(\tau|\MZ)$
\item And, there does not exist another trajectory  $\tau'$ of length less than $K$ that meets the same conditions.
\end{enumerate}
\end{defn}
That is, $K$ corresponds to the shortest possible trajectory that is still possible under $\AModel$, which may be better explained by $\MZ$ than \AModel. By restricting the defender's actions to a length less than $K$, the defender would never be able to select a path to a trap state that could potentially tip the beliefs of the attacker. (This even includes potential worst-case transitions that could occur in states where the defender chooses to use no-op actions.)


Now the question is how to calculate this budget $K$. A simple way to do this is to turn it into a uniform-cost search problem. Specifically, we will start the search from the initial state. At any state, the search can select one of the many transitions possible from that state, always expanding the prefix with the shortest length and ending the search when we have a trace that is more likely in \MZ~ than \AModel. There may be cases where such an upper bound does not exist or is too long to be of any practical importance to the defender. So we can stop the search, if we know that $K$ is going to be above some pre-specified limit. We can find this limit by solving the defender MDP ($\DModel$) without budget restrictions (and changing the reward function so it only reflects the number of steps needed before the attacker reaches a trap state) and looking at the optimal policy. This gives us the expected discounted number of steps the defender would need to get the optimum outcome, and if the $K$ is above that limit, then it does not really matter. If we are additionally guaranteed that at least one trap state is reachable from every state in the MDP, we can set the discount factor to $1$ to get a more accurate estimate of the steps.

%% file: Sections/5-evalution.tex
\section{Empirical Evaluation}

\begin{table*}[!t]
\centering
  \begin{tabular}{|r|c|c|c|c|}
    \toprule
    \multicolumn{2}{|c|}{Problem Instance} &
     Average Value of the attacker &
    Average Value of the defender & Average Defender  \\[0ex]
   
Task      & Instance&&&Planning Time (s)\\
          \midrule
    \multirow{4}{*}{Gridworld} &$4\times4$& 0.94&-0.32&0.46\\ 
&$6\times6$&0.89&-0.43&5.19\\
&$8\times8$&0.83&-0.33&74.68\\
&$9\times9$&0.83&-0.24&94.55\\
    \bottomrule
 \multirow{4}{*}{Four Rooms} &$4\times4$&0&0&0.013\\
&$6\times6$&0.85&-0.24&21.35\\
&$8\times8$&0&0&1.085\\
&$9\times9$&0.78&-0.48&62.51\\
    \bottomrule
 \multirow{4}{*}{Rock Sampling} &$4\times4$& 735.46&0&14.87\\ 
&$6\times6$&868.91&0&272.68 \\
&$8\times8$&-&-&-\\
&$9\times9$&-&-&-\\
    \bottomrule
 \multirow{4}{*}{Puddle}&$\delta=0.2$&529.08&0&0.14\\       
&$\delta=0.3$&534.42&0&0.07\\           
&$\delta=0.4$&547.72&0&0.04\\   
&$\delta=0.5$&547.72&0&0.04\\  
    \bottomrule
 
  \end{tabular}
\vspace{2pt}
\caption{
Empirical evaluation of the proposed algorithm on a number of standard MDP benchmarks}
\label{tab1}
\end{table*}
Our primary focus with these evaluations is to evaluate the computational characteristics of our proposed method on some standard MDP planning benchmarks. We chose some popular variations of a grid world domain to see how well our method performs as we vary the size and trap locations. We considered the following domains.
\begin{enumerate}
    \item Gridworld: This is the basic grid world domain. The agent’s goal is to navigate from a fixed initial state to a goal position. Its path to the goal may be blocked by walls and cells with lava pits. We tested our method on four different instances of grid world, with sizes of $4\times4$, $6\times6$, $8\times8$ and $9\times9$. For each grid instance, we tested our method five times using traps states generated at different random locations. The stochasticity in this domain is controlled via a slipping probability that causes the agent to move in a different direction with a small probability.
    \item Four Rooms: A popular variation of the grid world, in which the agent must navigate through different rooms by using the door connecting the rooms. This domain is particularly popular within the context of hierarchical RL literature \cite{sutton1999between}. As in the previous case, we used four instances of sizes  $4\times4$, $6\times6$, $8\times8$ and $9\times9$. Each instance was again tested on five randomly generated trap states. Interestingly, we found that the specific implementation of the four rooms problem generator we used created instances where the attacker could not reach the goal location for mazes of sizes  $4\times4$ and $8\times8$. The source of the stochasticity of the domain remains the same as the original grid world.
    \item Next, we considered an MDP version of the rock sampling domain introduced by \Shortcite{smith2012heuristic}. The agent takes the form of a rover that needs to go around collecting good rock samples from various points of the grid. We again consider grids of the same four sizes. We tested each instance five times on randomly generated traps. The source of the stochasticity of the domain remains the same as the original grid world. The original domain had a negative reward for sampling bad rocks, which we skipped for this experiment.
    \item Finally, we consider a puddle domain \cite{boyan1994generalization}, in which the agent needs to reach a goal state while avoiding water puddles. Our implementation used a fixed map of size $1\times1$ but allowed variations in the size of the step being taken by the agent. We looked at step sizes ($\delta$) of 0.2, 0.3, 0.4, and 0.5. Each step size was again tested five times over randomly selected trap positions. The source of the stochasticity was again the probability of slipping. In the original task definition, moving through a puddle causes a negative reward. However, to ensure that the agent only receives non-negative value, we added a positive reward for all actions that occur outside the puddle, while actions within the puddle received a reward of zero. While this could change the optimal behavior of the agent, it still remains a viable test domain as the defender could still use the stochasticity of the domain to lead the agent to specific trap states.
\end{enumerate}
   We used value iteration as our planning method and used the implementation provided by the \texttt{Simple\_RL} framework \cite{abel2019simple_rl}. We also used the task implementation provided by the framework, with minimal changes made to ensure the attacker’s reward was always positive and the source of the stochasticity was symmetric across the domains. We used a uniform cost search for the calculation of the defender budget. To keep the budget calculation simple, we kept an upper bound of 15 steps on the defender budget. If the search looked at paths of length higher than 15 steps, the search simply returned 15 as the defender budget. While this places a simple limit on the possible defender budget, ensuring the budget search ends in finite time, one could also calculate task-specific budget limits using the method discussed in Section \ref{sec:budget}.

   
   Table \ref{tab1} presents a summary of the results. All reported values are averaged across five instances. The slip probability for all instances was kept at 0.5. All experiments were run with a timeout of 30 minutes. The system used for testing was a 2GHz Quad-Core Intel Core i5, with 16 GB RAM, 3733 MHz, and LPDDR4X memory.

The primary points of comparison are the value obtained by the attacker if no defender were present and the value obtained by the defender. When the defender intervenes, the effective value obtained by the attacker is equal to the absolute value of the defender value. Our method decreases this value in almost every case. The only two instances that do not hold to this pattern are the ones where the attacker value is already zero. 

The instances with values equal to zero are ones in which the defender policy can guarantee that the attacker will be led to a trap state. We see this primarily in the rock sampling domain and the puddle domain. The defender's ability to induce policies that guarantee the attacker will be trapped depends on multiple factors, including the reachability of the trap state from the attacker's starting position, the possible policies that may be followed by the attacker, the potential stochasticity in the domain (so the defender can mask its actions), and the budget available to the defender. The domains where we saw more frequent defender policies that guaranteed entrapment were those with fewer obstacles that the attacker could not cross. In the case of the puddle, even the puddle does not prevent the agent from moving through there, but only causes a smaller reward. Similarly, in the case of the rock sampling domain, the agent is free to move through the map.

Another factor we were interested in measuring was the time taken by the defender to produce a policy. In theory, the policy generated by the defender would consider all possible behavior that the attacker could exhibit. In our current formulation, the defender planning time is not a major bottleneck. Yet it is still an interesting question to study, as our method relies on a complex compilation which considerably increases the state space of the problem instance. The results show that, for most cases, the system was able to generate policies within a reasonable amount of time for all domains except Rock Sampling. For Rock sampling, our approach timed for grids of sizes $8\times 8$ and $9\times 9$. However, the current planning framework we use (i.e., \texttt{Simple\_RL}) adds overhead to our approach, which may be avoidable in many MDP planning settings. \texttt{Simple\_RL} framework focuses on using a generative specification of the task and performs a separate sampling process to create an estimate of the explicit transition function. In cases where the model is given upfront, we could avoid the need to perform such model estimation operations.
As expected, for the problems where the planner succeeds, we see that the time taken increases with the size of the problem. Specifically, the defender model’s construction grows polynomially with the size of the state space. In the case of the puddle domain, it is worth noting that the effective size of the state space increases with a step size. So, the problem corresponding to step size $\delta=0.2$ corresponds to the largest problem and $\delta=0.5$ the smallest one. Table \ref{tab2} presents a breakdown of the time taken to construct the model and plan for the maximum environment size used in each domain (Gridworld: 9x9, Four Rooms: 9x9, Rock Sampling: 6x6, Puddle: 0.2), given a budget of 15. The system used for these tests was a 2.4GHz quad-core Intel Core i7 processor with 4GB 1333MHz DDR3 memory. While our current implementation is compatible with offline decision-making, a future goal for the project could be the use of a faster online planner to generate defender strategies. The code used for all experiments can be found at https://github.com/blcates/AttackerEntrapment.

\begin{table}[htbp]
    \begin{center}
        \begin{tabular}{|c|c|c|}
            \toprule
            Domain & \multicolumn{2}{|c|}{Time Taken (s)} \\
            \textbf{ }  & Model Construction & Planning \\
                \bottomrule
             Gridworld & 22.365 & 44.694\\
                \bottomrule
            Four Rooms & 14.927 & 35.693\\
                \bottomrule
            Rock Sampling & 23.774  & 5.431\\
                \bottomrule
            Puddle &  131.280  & 176.203\\
    \bottomrule
        \end{tabular}
    \end{center}
    \caption{Model construction and planning time}
    \label{tab2}
\end{table}

%% file: Sections/6-relatedwork.tex
\section{Related Work}
Much of the literature \cite{liu2021deceptive, masters2021extended,masters2020characterising,masters2017deceptive} around deceptive path planning involves an agent obfuscating its goal or path in order to mislead a passive observer. The defender described in our method turns this problem on its head. It is instead, an active observer attempting to deceive another agent. In their survey of agent interpretability, \Shortcite{chakraborti2019explicability} speculate that a fully active observer could have its “own goals and actions, with the ability to even assist or impede the agent from achieving its goals”. The latter is precisely what is achieved by our defender. To the best of our knowledge, this is the first work to implement an active observer agent whose influence and indeed, the very existence remains unknown to the agent it is deceiving. \looseness=-1

Deception \cite{bowyer1982cheating} requires that the deceiver must be able to influence the target’s actions, and it must be possible for the target to misread its situation \cite{davis2016deception}. \Shortcite{masters2017deceptive} makes a distinction between deception which uses simulation, and that which uses dissimulation. Simulation is generally described as ``showing the false.'' Whereas dissimulation, by contrast, is described as ``hiding the real.'' Put another way, telling a lie is a form of simulation, while a lie of omission is a form of dissimulation. The defender uses only dissimulation. It does not explicitly convey false information, such as convincing the attacker that the target is in a different location. It simply allows the attacker to maintain the false belief that the environment is not being controlled by an unseen adversary.

The defender can be further classified using the more specific categories of AI deception established by \Shortcite{masters2021extended}. It employs a strategy of calculating deception, in which the deceiver possesses more knowledge about the environment and exploits this knowledge to gain an advantage over the opponent. In this case, the asymmetry of knowledge is awareness of the defender’s existence and the presence of trap states. Unbeknownst to the attacker, the defender guides the attacker’s movements toward the equally unknown trap states.

It is useful to compare the defender’s method with those of other deceptive path-planning agents. \Shortcite{kulkarni2019unified} chose obfuscated paths for adversarial path planning scenarios  by conducting a search over nodes that represent each state along with the belief state of the adversary. Paths were chosen, in part, based on how they affected the  adversary’s belief state. Specifically, the planning agent aims to choose a path that does not make its goal known to an adversarial observer \cite{kulkarni2018resource, keren2016privacy,bernardini2020optimization,10.5555/3398761.3399017}. Likewise, our defender agent calculates its budget, and the number of actions it can take, by finding the shortest path to a trap state which does not violate the attacker’s model of the environment’s rules. 

Comparisons can also be drawn to \Shortcite{ Chakraborti2015serendipity} and \Shortcite{ Fern2014assistance}. The planning agents in both of these works consider the humans in their environment to be unaware of the agents’ existence.  Unlike the defender, these agents operate in cooperative scenarios. Thus, the human’s ignorance is not a necessary feature of the task. In \Shortcite{Fern2014assistance}, it is a simplifying assumption that the authors use to generate more effective solutions. Whether the case is collaborative or adversarial, the result of assuming ignorance of the AI’s existence is that the naïve agent does not alter its decisions in response to the AI’s presence. This simplification is what allows the attacker entrapment problem to be solved as an MDP, rather than as a more complex game-theoretic problem. In future work, it could be interesting to compare the defender's performance to that of comparable models which rely on game-theoretic solutions \cite{brafman2009planning,speicher2018stackelberg,tambe2011security}.

%% file: Sections/7-discussion.tex
\section{Conclusion and Discussion}

The primary contribution of this paper has been the introduction of the {\em attacker entrapment problem}, in which an attack is undermined by a covert defender. Both agents, each represented using a Markov Decision Process, select actions that maximize their value. The defender maximizes value by leading the attacker to choose actions that minimize its own value, ideally ending in a trap. By only choosing actions that do not alert the attacker to its presence, the defender is able to exploit its opponent’s ignorance of the environment, thereby thwarting the unwitting interloper. 

By demonstrating this interaction in four MDP benchmark domains, we have shown that a defender agent like the one described can effectively reduce the value obtained by an attacker. We have also presented a method for calculating a pessimistic lower bound on the number of intervening actions that a defender can take without altering an attacker’s belief about the environment. 

Future iterations of this research could investigate different planning strategies. 
Our current method calculates every possible state for every possible attacker action, which would be prohibitively slow for a large state space. 
As a next step, we hope to investigate the utility of online and approximate planners, such as variants of Monte Carlo Tree Search based planners. Online planning capabilities would make our approach more widely applicable in a variety of adversarial scenarios.

Past works have framed deceptive path planning in terms of an agent deceiving a passive observer \cite{kulkarni2019unified, ornik2018control}. The observer does not know the agent’s goal or plan, and it must continually update its own beliefs by inferring intent from the opponent’s actions. An interesting problem for future research would be to replace the passive observer in these scenarios with our defender agent. The defender could be modified to maintain an estimation of the attacker’s belief state, which can be updated based on the attacker’s actions. The defender could still conceal its presence, but the ability to handle evolving attacker beliefs would make it more robust if its presence became known. The defender could even make choices that risk violating the attacker’s concept of the environmental rules if doing so would ultimately lead to a more desirable outcome.

Such changes lead us toward a more realistic mental model for potential human attackers in a real-world setting. Another step toward a defender that could contend with a human attacker would be the exploitation of cognitive biases. Though humans often think of themselves as rational beings, they are not purely rational, and, in fact, their rationality fails in predictable ways. These predictable failings, called cognitive biases, have been categorized and well-documented in psychological literature. In future research, we plan to determine how common cognitive biases can be represented and detected in planning problems. A defender which could take advantage of these weaknesses would be equipped to deceive an opponent using more clever and subtle strategies.

It is our hope that the concepts we have introduced will provide fruitful avenues for future research. Through this simple example, we have demonstrated the potential utility of a covert defender in adversarial path-planning tasks. We are eager to further investigate this line of research and look forward to the new problems it will expose us to.

%% file: Sections/8-ack.tex
\section*{Acknowledgements}
\noindent The authors would like to thank Dr. Indrajit Ray for helpful discussions that contributed to the inception of this project.

%% file: main.bbl
\begin{thebibliography}{30}
\providecommand{\natexlab}[1]{#1}

\bibitem[{Abel(2019)}]{abel2019simple_rl}
Abel, D. 2019.
\newblock simple\_rl: Reproducible Reinforcement Learning in Python.
\newblock In \emph{RML@ ICLR}.

\bibitem[{Bell and Whaley(1982)}]{bowyer1982cheating}
Bell, J.~B.; and Whaley, B. 1982.
\newblock \emph{Cheating: Deception in war \& magic, games \& sports, sex \&
  religion, business \& con games, politics \& espionage, art \& science}.
\newblock St Martin's Press.

\bibitem[{Bernardini, Fagnani, and Franco(2020)}]{bernardini2020optimization}
Bernardini, S.; Fagnani, F.; and Franco, S. 2020.
\newblock An Optimization Approach to Robust Goal Obfuscation.
\newblock In \emph{Proceedings of the International Conference on Principles of
  Knowledge Representation and Reasoning}, volume~17, 119--129.

\bibitem[{Boyan and Moore(1994)}]{boyan1994generalization}
Boyan, J.; and Moore, A. 1994.
\newblock Generalization in reinforcement learning: Safely approximating the
  value function.
\newblock \emph{Advances in neural information processing systems}, 7.

\bibitem[{Brafman et~al.(2009)Brafman, Domshlak, Engel, and
  Tennenholtz}]{brafman2009planning}
Brafman, R.~I.; Domshlak, C.; Engel, Y.; and Tennenholtz, M. 2009.
\newblock Planning Games.
\newblock In \emph{IJCAI}, 73--78. Citeseer.

\bibitem[{Chakraborti et~al.(2015)Chakraborti, Briggs, Talamadupula, Zhang,
  Scheutz, Smith, and Kambhampati}]{Chakraborti2015serendipity}
Chakraborti, T.; Briggs, G.; Talamadupula, K.; Zhang, Y.; Scheutz, M.; Smith,
  D.; and Kambhampati, S. 2015.
\newblock Planning for serendipity.
\newblock In \emph{2015 IEEE/RSJ International Conference on Intelligent Robots
  and Systems (IROS)}, 5300–5306.

\bibitem[{Chakraborti et~al.(2019)Chakraborti, Kulkarni, Sreedharan, Smith, and
  Kambhampati}]{chakraborti2019explicability}
Chakraborti, T.; Kulkarni, A.; Sreedharan, S.; Smith, D.~E.; and Kambhampati,
  S. 2019.
\newblock Explicability? legibility? predictability? transparency? privacy?
  security? the emerging landscape of interpretable agent behavior.
\newblock In \emph{Proceedings of the international conference on automated
  planning and scheduling}, volume~29, 86--96.

\bibitem[{Davis(2016)}]{davis2016deception}
Davis, A.~L. 2016.
\newblock Deception in game theory: a survey and multiobjective model.
\newblock Technical report, Air Force Institute of Technology Wright-Patterson
  AFB OH.

\bibitem[{Fern et~al.(2014)Fern, Natarajan, Judah, and
  Tadepalli}]{Fern2014assistance}
Fern, A.; Natarajan, S.; Judah, K.; and Tadepalli, P. 2014.
\newblock A Decision-Theoretic Model of Assistance.
\newblock \emph{Journal of Artificial Intelligence Research}, 50: 71–104.

\bibitem[{Jeon, Milli, and Dragan(2020)}]{noisyrat}
Jeon, H.~J.; Milli, S.; and Dragan, A.~D. 2020.
\newblock Reward-rational (implicit) choice: {A} unifying formalism for reward
  learning.
\newblock In \emph{Advances in Neural Information Processing Systems 33: Annual
  Conference on Neural Information Processing Systems 2020, NeurIPS 2020,
  December 6-12, 2020, virtual}.

\bibitem[{Kaelbling, Littman, and Cassandra(1998)}]{KaelblingLC98}
Kaelbling, L.~P.; Littman, M.~L.; and Cassandra, A.~R. 1998.
\newblock Planning and Acting in Partially Observable Stochastic Domains.
\newblock \emph{Artif. Intell.}, 101(1-2): 99--134.

\bibitem[{Keren, Gal, and Karpas(2016)}]{keren2016privacy}
Keren, S.; Gal, A.; and Karpas, E. 2016.
\newblock Privacy Preserving Plans in Partially Observable Environments.
\newblock In \emph{IJCAI}, 3170--3176.

\bibitem[{Keren, Gal, and Karpas(2021)}]{keren2021goal}
Keren, S.; Gal, A.; and Karpas, E. 2021.
\newblock Goal recognition design-survey.
\newblock In \emph{Proceedings of the Twenty-Ninth International Conference on
  International Joint Conferences on Artificial Intelligence}, 4847--4853.

\bibitem[{Kulkarni et~al.(2018)Kulkarni, Klenk, Rane, and
  Soroush}]{kulkarni2018resource}
Kulkarni, A.; Klenk, M.; Rane, S.; and Soroush, H. 2018.
\newblock Resource bounded secure goal obfuscation.
\newblock In \emph{AAAI Fall Symposium on Integrating Planning, Diagnosis and
  Causal Reasoning}.

\bibitem[{Kulkarni, Srivastava, and Kambhampati(2019)}]{kulkarni2019unified}
Kulkarni, A.; Srivastava, S.; and Kambhampati, S. 2019.
\newblock A unified framework for planning in adversarial and cooperative
  environments.
\newblock In \emph{Proceedings of the AAAI Conference on Artificial
  Intelligence}, volume~33, 2479--2487.

\bibitem[{Kulkarni, Srivastava, and
  Kambhampati(2020)}]{10.5555/3398761.3399017}
Kulkarni, A.; Srivastava, S.; and Kambhampati, S. 2020.
\newblock Signaling Friends and Head-Faking Enemies Simultaneously: Balancing
  Goal Obfuscation and Goal Legibility.
\newblock In \emph{Proceedings of the 19th International Conference on
  Autonomous Agents and MultiAgent Systems}, AAMAS '20, 1889–1891. Richland,
  SC: International Foundation for Autonomous Agents and Multiagent Systems.
\newblock ISBN 9781450375184.

\bibitem[{Letchford and Vorobeychik(2013)}]{letchford2013optimal}
Letchford, J.; and Vorobeychik, Y. 2013.
\newblock Optimal interdiction of attack plans.
\newblock In \emph{AAMAS}, 199--206.

\bibitem[{Liu et~al.(2021)Liu, Yang, Miller, and Masters}]{liu2021deceptive}
Liu, Z.; Yang, Y.; Miller, T.; and Masters, P. 2021.
\newblock Deceptive reinforcement learning for privacy-preserving planning.
\newblock \emph{arXiv preprint arXiv:2102.03022}.

\bibitem[{Masters, Kirley, and Smith(2021)}]{masters2021extended}
Masters, P.; Kirley, M.; and Smith, W. 2021.
\newblock Extended goal recognition: a planning-based model for strategic
  deception.
\newblock In \emph{Proceedings of the 20th International Conference on
  Autonomous Agents and MultiAgent Systems}, 871--879.

\bibitem[{Masters and Sardina(2017)}]{masters2017deceptive}
Masters, P.; and Sardina, S. 2017.
\newblock Deceptive Path-Planning.
\newblock In \emph{IJCAI}, 4368--4375.

\bibitem[{Masters et~al.(2020)Masters, Smith, Sonenberg, and
  Kirley}]{masters2020characterising}
Masters, P.; Smith, W.; Sonenberg, L.; and Kirley, M. 2020.
\newblock Characterising Deception in AI: A Survey.
\newblock In \emph{Deceptive AI}, 3--16. Springer.

\bibitem[{Ornik and Topcu(2018)}]{ornik2018control}
Ornik; and Topcu. 2018.
\newblock Deception in Optimal Control.
\newblock In \emph{56th Annual Allerton Conference on Communication, Control,
  and Computing}, 821–828.

\bibitem[{Puterman(1990)}]{puterman1990markov}
Puterman, M.~L. 1990.
\newblock Markov decision processes.
\newblock \emph{Handbooks in operations research and management science}, 2:
  331--434.

\bibitem[{Senator(2019)}]{senator2019science}
Senator, M.~T. 2019.
\newblock Science of Artificial Intelligence and Learning for Open-world
  Novelty (SAIL-ON).

\bibitem[{Smith and Simmons(2012)}]{smith2012heuristic}
Smith, T.; and Simmons, R. 2012.
\newblock Heuristic search value iteration for POMDPs.
\newblock \emph{arXiv preprint arXiv:1207.4166}.

\bibitem[{Speicher et~al.(2018)Speicher, Steinmetz, Backes, Hoffmann, and
  K{\"u}nnemann}]{speicher2018stackelberg}
Speicher, P.; Steinmetz, M.; Backes, M.; Hoffmann, J.; and K{\"u}nnemann, R.
  2018.
\newblock Stackelberg planning: Towards effective leader-follower state space
  search.
\newblock In \emph{Proceedings of the AAAI Conference on Artificial
  Intelligence}, volume~32.

\bibitem[{Sreedharan et~al.(2021)Sreedharan, Kulkarni, Smith, and
  Kambhampati}]{Sreedharan0SK21}
Sreedharan, S.; Kulkarni, A.; Smith, D.~E.; and Kambhampati, S. 2021.
\newblock A Unifying Bayesian Formulation of Measures of Interpretability in
  Human-AI Interaction.
\newblock In \emph{Proceedings of the Thirtieth International Joint Conference
  on Artificial Intelligence, {IJCAI} 2021, Virtual Event / Montreal, Canada,
  19-27 August 2021}, 4602--4610. ijcai.org.

\bibitem[{Sutton, Precup, and Singh(1999)}]{sutton1999between}
Sutton, R.~S.; Precup, D.; and Singh, S. 1999.
\newblock Between MDPs and semi-MDPs: A framework for temporal abstraction in
  reinforcement learning.
\newblock \emph{Artificial intelligence}, 112(1-2): 181--211.

\bibitem[{Tambe(2011)}]{tambe2011security}
Tambe, M. 2011.
\newblock \emph{Security and game theory: algorithms, deployed systems, lessons
  learned}.
\newblock Cambridge university press.

\bibitem[{Zabell(1992)}]{zabell1992predicting}
Zabell, S.~L. 1992.
\newblock Predicting the unpredictable.
\newblock \emph{Synthese}, 90(2): 205--232.

\end{thebibliography}
